\def\year{2021}\relax
\documentclass[letterpaper]{article} 
\usepackage{aaai21}  
\usepackage{times}  
\usepackage{helvet} 
\usepackage{courier}  
\usepackage[hyphens]{url}  
\usepackage{graphicx} 
\usepackage{natbib}  
\usepackage{caption} 
\usepackage{tabularx}
\usepackage{algorithm,algorithmic}
\usepackage{amsmath}
\usepackage{amsfonts}
\usepackage{amsbsy}
\usepackage{amsthm}
\usepackage{xspace}
\usepackage{enumitem}
\usepackage{graphicx}
\usepackage{indentfirst}
\usepackage{todonotes}
\usepackage{color}
\usepackage{multirow}
\usepackage{booktabs}

\DeclareMathOperator{\dist}{dist}

\DeclareMathOperator{\vol}{Vol}
\DeclareMathOperator{\area}{Area}

\newcommand{\R}{\mathbb{R}}
\newcommand{\rd}{\delta}
\newtheorem{definition}{Definition}

\newtheorem{theorem}{Theorem}
\newcommand{\RE}{\mathbb{R}}

\newcommand{\calB}{\mathcal{B}}
\newcommand{\calF}{\mathcal{F}}

\newcommand{\calC}{\mathcal{C}}
\newcommand{\calP}{\bar{p}}
\newcommand{\calV}{\mathcal{V}}
\newcommand{\calU}{\mathcal{U}}

\newcommand{\calS}{\mathcal{S}}

\newcommand{\chij}{\chi_{t_0}^{t_j}}
\newcommand{\chit}{\chi_{t_0}^{t}}

\newcommand{\Lp}{L}
\newcommand{\lambdaS}[1][\varphi]{\lambda_{\Sigma_{#1}}}
\newcommand{\rmu}[1][\varphi]{r_{{#1}}}
\newcommand{\niton}{\not\owns}
\newcommand{\phistar}{\varphi^\star}

\newcommand{\Ystar}[1][\varphi]{\calC(r_{{#1}})}
\newcommand{\Bphi}[1][\varphi]{B({#1},r_{#1})^S}

\newcommand{\Lc}{\dist}

\title{On the Verification of Neural ODEs with Stochastic Guarantees}
\author {
    Sophie Gruenbacher\textsuperscript{\rm 1},
    Ramin Hasani\textsuperscript{\rm 1,2},
    Mathias Lechner\textsuperscript{\rm 3},
    Jacek Cyranka\textsuperscript{\rm 4},\\
    Scott A. Smolka\textsuperscript{\rm 5},
    Radu Grosu\textsuperscript{\rm 1}\\
}
\affiliations {
    \textsuperscript{\rm 1} Technische Universit\"at Wien (TU Wien)\\
    \textsuperscript{\rm 2} Massachusetts Institute of Technology (MIT)\\
    \textsuperscript{\rm 3} Institute of Science and Technology Austria (IST Austria)\\
    \textsuperscript{\rm 4} University of Warsaw\\
    \textsuperscript{\rm 5} Stony Brook University\\
}

\begin{document}
\maketitle              
\begin{abstract}
We show that \emph{Neural} ODEs, an emerging class of time-continuous neural networks, can be verified by solving a set of global-optimization problems. For this purpose, we introduce \emph{Stochastic Lagrangian Reachability} (SLR), an abstraction-based technique for constructing a tight \emph{Reachtube} (an over-approximation of the set of reachable states over a given time-horizon), and provide stochastic guarantees in the form of confidence intervals for the Reachtube bounds. 
SLR inherently avoids the infamous wrapping effect (accumulation of over-approximation errors) by performing local optimization steps to expand safe regions instead of repeatedly forward-propagating them as is done by deterministic reachability methods. 
To enable fast local optimizations, we introduce a novel forward-mode adjoint sensitivity method to compute gradients without the need for backpropagation.
Finally, we establish asymptotic and non-asymptotic convergence rates for SLR.
\end{abstract}
%
%
%
\section{Introduction}\label{sect:introduction}
Neural ordinary differential equations (Neural ODEs) \cite{neuralODEs},  which are analogous to a continuous-depth version of deep residual networks \cite{he2016deep},  exhibit considerable computational
efficiency
on time-series modeling tasks. Although Neural ODEs do not necessarily improve the performance of contemporary deep models, they enable the rich theory and tools from the field of differential equations to be applied to deep models. Examples include a better characterization of Neural ODEs \cite{rubanova2019latent,dupont2019augmented,durkan2019neural,jia2019neural}, and a better understanding of their robustness \cite{yan2020robustness},  stability \cite{yang2020dynamical},
and controllability \cite{quaglino2019snode,holl2020learning,kidger2020neural}.

As the use of Neural ODEs on real-world applications increases \cite{finlay2020train,lechner2020neural,erichson2020lipschitz,lechner2020learning,hasani2020natural}, so does the importance of ensuring their safety through the use of verification techniques. In this paper, we establish a theoretical foundation for the verification of Neural ODE networks. 

In particular, we introduce \emph{Stochastic Lagrangian Reachability} (SLR), a new analysis technique with provable convergence and conservativeness guarantees for Neural ODEs $\partial_t x\,{=}\,f$, with field $f(x,x(0),t,\theta)$, hidden states $x(t)$, and parameters $\theta$. (SLR works in fact for any nonlinear system defined by a set of nonlinear differential equations.)

At the core of SLR is the translation of the reachability problem to a global optimization problem, at every time step $t$. The latter is solved globally, by uniformly sampling states $x$ from an initial ball $\calB_0$, and locally, by computing a local minimum via gradient descent from $x$. SLR avoids gradient descent if $x$ is within a spherical-cap around a previously sampled state or its corresponding local minimum. 

The radius of the cap is derived from the interval computation of the local Lipschitz constant of the objective function within the cap. The minimum computed by SLR at time $t$ stochastically defines an as-tight-as-possible ellipsoid covering all states reached at $t$ by the solution starting in $\calB_0$, with tolerance $\mu$ and confidence $1\,{-}\,\gamma$, for given values of $\mu$ and $\gamma$.  See Figure~\ref{fig:Notation}.

\begin{figure}[t]
    \centering
    \includegraphics[width=0.9\columnwidth]{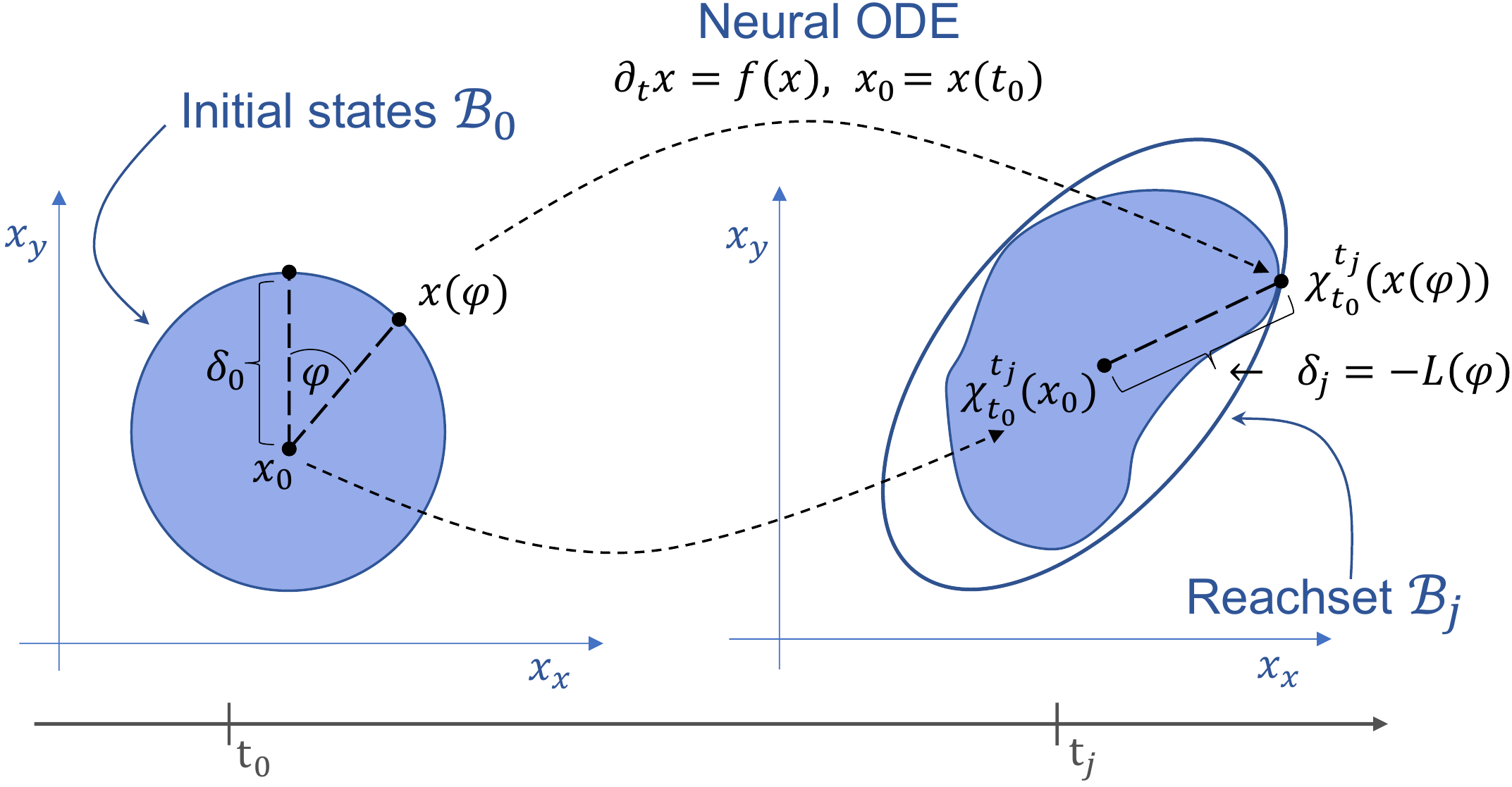}
    \caption{The conservative reachset $\calB_j$ at time $t_j$ computed using Lagrangian reachability and global optimization, for a Neural ODE starting from the ball $\calB_0$ at time $t_0$.}
    \label{fig:Notation}
\end{figure}

Since SLR employs interval arithmetic only locally to compute the spherical-caps (also called safety or tabu regions), it avoids the infamous wrapping effect \cite{lohnerOrig} of deterministic reachability methods (see Table~\ref{tab:related_works}), which prevents them from being deployed in practice. Consequently, our approach scales up to large-scale, real-life Neural ODEs.  To the best of our knowledge, none of the available tools has been successfully applied to Neural ODEs. 

We also introduce a novel forward formulation of the adjoint sensitivity method \cite{pontryagin2018mathematical} to compute the loss gradients in the optimization flow. This enables us to improve the time complexity of the optimization process compared to similar methods \cite{neuralODEs,zhuang2020adaptive}. 
%
%
%
%
%

\noindent\textbf{Summary of results.} In this work, we present a thorough theoretical approach to the problem of providing safety guarantees for the class of time-continuous neural networks formulated as Neural ODEs. As the main result, we develop SLR, a differentiable stochastic Lagrangian reachability framework, formulated as a global optimization problem. In particular, we prove that 
SLR
converges (Theorem~\ref{thm:convergence guarantee}) to tight ellipsoidal safe regions (Theorem~\ref{thm:safety region radius}), within $\mathcal{O}(-\ln\gamma (\rd_0/r_{bound})^{2n})$ number of iterations (Theorem~\ref{thm:convergence rate}). This implies that for a given confidence level $\gamma$, our algorithm terminates according to the proposed rate, which leads to the important conclusion that the problem of constructing an ellipsoid abstraction of the true reachsets with probabilistic guarantees for Neural ODEs is decidable (the computed abstraction is conservative with confidence $\gamma$). We summarize our key contributions as follows:

\begin{itemize}
\itemsep0em
\item We introduce a theoretical framework for the verification of Neural ODEs by restating the reachability problem as a set of global-optimization problems.  
\item We solve each optimization problem globally, via uniform sampling, and locally, through gradient descent (GD), thereby avoiding costly Hessian computations in the process.
\item GD is avoided in spherical-caps around the start/end states of previous searches. The cap radius is derived from its local Lipschitz constant, computed via interval arithmetic.  
\item We design a forward-mode GD algorithm 
based on the adjoint sensitivity method for (Neural) ODEs.
\item We prove convergence properties of SLR, its safety guarantees, and discuss its time and space complexity. 
\end{itemize}

\section{Related Work}
\noindent\textbf{Global optimization.} The literature on global optimization for continuous problems is vast and includes many different approaches depending on the smoothness assumptions made about the objective function. Evolutionary strategies like those based on the covariance matrix \cite{cma,cmaes} work for general continuous objectives.  Deterministic interval-based branch-and-bound methods \cite{neumaier_2004,globintervals} work for differentiable objectives, and Lipschitz global optimization \cite{piyavskii,schubert,lipschitz} for  objectives satisfying the Lipschitz condition.
Our work is closest to the BRST algorithm \cite{brst,brst2,brst3} which for smooth objectives uses Hessians to compute the basins of attraction for local minima as ellipsoidal bounds. Such basins define tabu regions. The final estimate for the global minimum and reasonable confidence bounds are provided.\\[1mm]
\noindent\textbf{Stochastic reachability.} Existing work is mainly concerned with the verification of safety guarantees for stochastic hybrid systems with continuous dynamics (ODEs) in each mode.
Stochasticity is introduced in several ways: uncertainty in the model parameters \cite{sreach,sisat,probreach}, uncertainty in the discrete jumps between modes \cite{e71eb65e23844408b72fe95a84f88cb6}, and uncertainty in the initial state \cite{10.1145/3126508}.
The work of \cite{reliablecomput} focuses on the probabilistic verification of continuous-time ODEs with uncertainty in parameters and initial states.

\noindent\textbf{Reachability for continuous dynamical systems.} Most of the relevant techniques are deterministic and based on interval arithmetic. We provide a qualitative summary of existing reachability methods for continuous-time systems in Table~\ref{tab:related_works}.

\begin{table*}[t]
\scriptsize
\centering
\caption{A Perspective on Related Work}
\vspace{-2mm}
\begin{tabular}{l|c|c|c|c}
\toprule
\textbf{Technique} & \textbf{Deterministic} & \textbf{Parallelizable (single step)} & \textbf{Basis} & \textbf{wrapping effect} \\
\midrule
LRT \cite{Cyranka2017} & yes & no & Infinitesimal strain theory & yes \\
CAPD \cite{CAPD}& yes & no & Lohner algorithm & yes \\
Flow-star \cite{flowstar}& yes & no & Taylor models & yes \\
$\delta$-reachability \cite{deltadecidable} & yes & no & approximate satisfiability & yes\\
C2E2  \cite{c2e2} & yes & no & discrepancy function & yes\\
LDFM \cite{fansimul}& yes & yes & simulation, matrix measures & no\\
TIRA \cite{tira}  & yes & yes & second-order sensitivity & no\\
Isabelle/HOL \cite{isabelle} & yes & no & proof-assistant & yes\\
Breach \cite{breach,donze}& yes & yes & simulation, sensitivity & no\\
PIRK \cite{pirk} & yes & yes & simulation, contraction bounds & no\\
HR \cite{hr} & yes & no & hybridization & yes\\
ProbReach \cite{probreach2} & no & no & $\delta$-reachability, probability interval & yes \\
VSPODE \cite{reliablecomput} & no &no & p-boxes & yes \\
GP \cite{gp} & no & no & Gaussian process& no \\
SLR \textbf{Ours} & no & yes & stochastic Lagrangian reachability & no\\
\bottomrule
\end{tabular}
\caption*{\footnotesize \textbf{Note:} Deterministic refers to approaches that provide an overapproximation of the reach-set without any uncertainties. A ``No” in the deterministic column indicates a stochastic approach that yields a reach-set with a corresponding confidence interval.}
 \label{tab:related_works}
 \vspace{-4mm}
\end{table*}

\section{Setup}\label{sect:preliminaries}

In this section, we introduce our notation, preliminary concepts, and definitions required to construct our theoretical setup for the verification of Neural ODEs. 

\noindent\textbf{Neural ODE.} 
The derivative of the hidden states $x$ is computed by a neural network $f$ parameterized by $\theta$ as follows \cite{neuralODEs}: 
\begin{equation}
    \partial_t x = f(x,x(0),t,\theta), x_0 \in \calB_0
    \label{neuralode}
\end{equation}
We require that the Neural ODE is Lipschitz-continuous and forward-complete. The solution to this initial-value problem can be computed by numerical ODE solvers, from any initial system state $x(0)\,{=}\,x_0$. Consequently, the numerical solution can be trained by reverse-mode automatic differentiation \cite{rumelhart1986learning}, either through the solver, by a vanilla backpropagation algorithm \cite{hasani2020liquid}, or by treating the solver as a blackbox and using the adjoint sensitivity method \cite{pontryagin2018mathematical}. 

\noindent\textbf{Geometrical deformation in time by a flow $\chi$.} To describe the optimization problem, we use
Eulerian and Lagrangian coordinates from classical continuum mechanics. We regard the set of initial states, which is the ball $\calB_0\,{=}\,B(x_0,\rd_0)$, as a body that is being deformed in time by a flow $\chi$. Given a point $x\,{\in}\,\calB_0$ in Eulerian coordinates (the undeformed configuration), there is at every time $t_j\,{>}\,t_0$ the representation $x(t_j)\,{=}\,\chij(x)$ of that point in Lagrangian coordinates (the configuration deformed by $\chi$).

The deformation of $\calB_0$ in time is related to the Neural ODE, where $\chi$ is defined as the solution flow of Eq.~\eqref{neuralode}.

\noindent\textbf{Reachset.} A reachset is the set of all states reached at a target time $t$, given the initial states and a flow. More formally:
\begin{definition}
    Given a set of initial states $\calB_0$ at time $t_0$, the target time $t_j\,{\ge}\,t_0$, and the flow $\chi$ of the Neural ODE~\eqref{neuralode}, we call $\calB_j(\calB_0)\,{\subset}\,\R^n$ a conservative \emph{reachset} enclosure if $\chij(x)\,{\in}\, \calB_j(\calB_0)$, for all $x\,{\in}\, \calB_0$; i.e., the reachset bounds all state-trajectories of the Neural ODE.
\end{definition}
Whenever the initial set $\calB_0$ is known from the context, we
simply refer to the reachset as \textit{the Reachset at time $t_j$}, or $\calB_j$.

\noindent\textbf{Reachtube.} A reachtube is a series of reachsets within a determined time-horizon. Formally:

\begin{definition}\label{def:Reachtube}
    Given a set of initial states $\calB_0$ at time $t_0$, and a time horizon $T$, we use $B(\calB_0,T)$ to denote a sequence of time-stamped reachsets $\calB_1$, $\dots$, $\calB_k$ with $t_0 \,{\le}\, t_1 \,{\le}\,{\dots} \,{\le}\, t_k\,{=}\,T$.
\end{definition}
Whenever the initial set, time horizon, and flow are known from the context, we use the term \emph{reachtube over-approximation} or $\calB$, for that sequence of reachsets.

\begin{definition}[Ellipsoid]\label{def:ellipsoid}
    Given $A_j, M_j \in \R^{n\times n}$, $M_j\succ 0$ with $A_j^T A_j = M_j$ and $\lVert x \rVert_{M_j} = \sqrt{x^TM_jx}$, we call $B_{M_j}(x_0,\rd)$ a ball in metric $M_j$ (or an ellipsoid) with center $x_0$ and radius $\rd$ if $\lVert x-x_0\rVert_{M_j}\le\rd$ for all $x\,{\in}\, B_{M_j}(x_0,\rd)$.
\end{definition}

\noindent\textbf{Reachability as an optimization problem.} Given a time horizon $T$, an initial ball $\calB_0 \,{=}\, B_I(x_0, \rd_0)$ with center $x_0$ and radius $\rd_0$, and Euclidean metric $M_0\,{=}\,I$, our  goal is to find a tight reachtube $\calB$, bounding all state-trajectories of the Neural ODE~\eqref{neuralode}.

We capture the reachsets of $\calB$ by ellipsoids $\calB_j\,{=}\, B_{M_j}(\chij(x_0),\rd_j)$ with center $\chij(x_0)$, radius $\rd_j$, and metric $M_j$. At every time $t_j$, we use as the center $\chij(x_0)$, the numerical integration of $x_0$, and as the metric $M_j$, the optimal metric in $\chij(x_0)$ minimizing the volume of the ellipsoid, as proposed in~\cite{gruenbacher2020lagrangian}. 

Thus, our goal is to find at every time step $t_j$, a radius $\rd_j$ which (stochastically) guarantees that $\calB_j$ is a conservative reachset. I.e., at each $t_j$, we want to find the maximal distance of all $\chij(x)$ to center $\chij(x_0)$ in metric $M_j$ for $x\,{\in}\,\calB_0$, and define $\rd_j$ as this distance. Thus the optimization problem can be defined as follows:
\begin{align}
    \rd_j &\ge \max_{x\in\calB_0} \left\lVert \chij(x) - \chij(x_0)\right\rVert_{M_j}\label{eq:optim1}
    = \max_{x\in\calB_0} \dist\left(\chij(x)\right)
\end{align}
where we use $\dist(\chij(x))$ to describe the distance in Eq.~\eqref{eq:optim1} when metric $M_j$ and starting point $x_0$ are known.

As we require Lipschitz-continuity and forward-completeness, the map $x \mapsto \chij(x)$ is a homeomorphism and commutes with closure and interior operators. In particular, the image of the boundary of the set $\calB_0$ is equal to the boundary of the image $\chij(\calB_0)$. Thus, Eq.~\eqref{eq:optim1} has its optimum on the surface of the initial ball $\calB_0^S = \textrm{surface}(\calB_0)$, and we will only consider points on the surface. In order to be able to optimize this problem, we describe the points on the surface with (n-dimensional) polar coordinates such that every point $x\,{\in}\,\calB_0^S$ is represented by a tuple $(\rd_0,\varphi)$, with angles $\varphi \,{=}\, (\varphi_1,\dots,\varphi_{n-1})$ and center $x_0$, having a conversion function $x((\rd_0,\varphi),x_0)$ from polar to Cartesian coordinates, defined as follows:
\begin{equation}\label{eq:polar}
    \begin{aligned}
        &x((\rd_0,\varphi),x_0) = \\
        &\begin{pmatrix}
            x_{0,1} + \rd_0 \cos(\varphi_1)\\
            \vdots\\
            x_{0,n-1} + \rd_0 \sin(\varphi_1)\cdot\ldots\cdot\sin(\varphi_{n-2})\cos(\varphi_{n-1})\\
            x_{0,n} + \rd_0 \sin(\varphi_1)\cdot\ldots\cdot\sin(\varphi_{n-2})\sin(\varphi_{n-1})\\
        \end{pmatrix}
    \end{aligned}
\end{equation}
Whenever the center $x_0$ and the radius $\rd_0$ of the initial ball $\calB_0$ are known from the context, we will use the following notation: $x(\varphi)$ for the conversion from polar to Cartesian coordinates and $\varphi(x)$ for Cartesian to polar. Using polar coordinates, we restate the optimization problem~\eqref{eq:optim1} as follows:
\begin{align}\label{eq:optim2}
    \delta_j &= \max_{x\in\calB_0} \left\lVert \chij(x) - \chij(x_0)\right\rVert_{M_j}\nonumber\\
    &= \max_{\varphi\in \RE^{n-1}} \underbrace{\left\lVert \chij(x(\varphi)) - \chij(x_0)\right\rVert_{M_j}}
    _{=-\Lp(\varphi)}\nonumber\nonumber\\
    &=\min_{\varphi\in \RE^{n-1}} \Lp(\varphi) = m^\star,
\end{align}
We call $\Lp$ the \emph{loss function} in polar coordinates at time $t_j$ that we would like to minimize. Note that $\Lp$ also depends on the initial radius $\delta _0$ and initial center $x_0$; as these are fixed inputs, we do not consider them in the notation.
\section{Main Results}
In this section, we present our verification framework for Neural ODEs, which we call \textbf{Stochastic Lagrangian Reachability (SLR)}.
As the main results of this paper, we show that the algorithm guarantees safety and converges to the tightest ellipsoid, almost surely, in the limit of the number of samples. We then compute the convergence rate and discuss space and time complexities. 

\begin{algorithm}[t]
    \caption{Finding the local minimum}
    \label{algorithm:gradient descent}
    \begin{algorithmic}[1]
    \REQUIRE target time $t_j$, termination tolerance $\epsilon > 0$, learning rate $\gamma > 0$, initial guess $\varphi\in\RE^{n-1}$, loss function $L$, gradient of loss $\nabla_\varphi L$ 
    \STATE $l \leftarrow L(\varphi)$, $l_{prev} \leftarrow \infty$
    \WHILE{$|l-l_{prev}|/|l_{prev}| > \epsilon$}
        \STATE \textbf{compute} $\nabla_\varphi L$\label{line:loss gradient}
        \STATE $\varphi \leftarrow \varphi - \alpha \nabla_\varphi L$
        \STATE $l_{prev} \leftarrow l$
        \STATE $l \leftarrow L(\varphi)$
    \ENDWHILE
    \RETURN $\varphi, l$ 
    \end{algorithmic}
\end{algorithm}

\subsection{Gradient Computation}
Our algorithm uses gradient descent locally when solving the global optimization problem of Eq.~\eqref{eq:optim2}. Gradient descent is started from uniformly sampled points, which are not contained in already constructed safety regions.

Uniform sampling is used to repeatedly select an initial point from the surface of the ball $\calB_0$. Gradient descent is used from this point to find a local minimum. SLR is inspired by the \emph{gradient-only tabu-search} (GOTS) proposed in~\cite{stepanenko}. Instead of tabu regions, we use \emph{safety radii} $r(\varphi)$ to construct an area around already visited points $\varphi$, where we know for sure what the minimum value inside that region is. In the following, we describe the computational steps of the loss's gradient for the main SLR algorithm in greater detail. 

Given the target time $t_j$, termination tolerance $\epsilon \,{>}\, 0$, learning rate $\gamma\,{>}\, 0$, initial guess $\varphi\in\RE^{n-1}$, and loss function $L$, we seek to compute the gradient of loss $\nabla_\varphi L$. We introduce a new framework to compute the loss's gradient which is needed in Line~\ref{line:loss gradient} of Algorithm~\ref{algorithm:gradient descent} to find the local minimum. Using the chain rule, we can express the gradient $\nabla_\varphi L$ as follows:
\begin{equation}
\label{eq:derivatives}
\begin{split}
    &\frac{\partial L(\cdot)}{\partial \varphi}(\varphi) =  -\left.\frac{\partial \Lc \circ \chij  \circ x (\cdot)}{\partial\varphi}\right.\\
    & = - \underbrace{\left.\frac{\partial \Lc}{\partial y}\right|_{y=\chij( x(\varphi))}}_{(a)}
    \cdot
    \underbrace{\left.\frac{\partial \chij}{\partial x}\right|_{x=x(\varphi)}}_{(c)}
    \cdot
    \underbrace{\frac{\partial x(\cdot)}{\partial \varphi}}_{(b)}
 \end{split}
\end{equation}

\textbf{Part (a) - loss gradient wrt $y$:} The differentiation of the loss function defined in Eq.~\eqref{eq:optim1}
can be expressed as
\begin{align}\label{eq:loss gradient x}
    \partial_y \Lc (y) = A_j(y-\chij(x_0))\Lc(y)^{-1}A_j,
\end{align}
with $A_j$ from Def.~\ref{def:ellipsoid} and $M_j$ as the metric in $\chij(x_0)$ minimizing the volume of the ellipsoid~\cite{gruenbacher2020lagrangian}.

\textbf{Part (b) - polar gradient:} $ x(\varphi)$ describes the transformation from polar coordinates to Cartesian coordinates, as given in Eq.~\eqref{eq:polar}. The differentiation with respect to $\varphi$ is straightforward to obtain using the product rule and the derivatives of $\sin$ and $\cos$:
\begin{equation}\label{eq:polar gradient}
    \begin{aligned}
        &\partial_\varphi x(\varphi) = \\
        &\begin{pmatrix}
            -\rd_0 \sin(\varphi_1)\\
            \rd_0 \left(
            \cos(\varphi_1)\cos(\varphi_2) - \sin(\varphi_1)\sin(\varphi_2)
            \right)\\
            \vdots\\
        \end{pmatrix}
    \end{aligned}
\end{equation}

\textbf{Part (c) - gradient of the flow:}
The partial derivative $\partial_x \chij (x)$ in $x$ of the Neural ODE solution flow $\chi$ with respect to the initial condition is called the gradient of the flow or \emph{deformation gradient} in~\cite{linTE,contMec}, and the \emph{sensitivity matrix} in~\cite{breach,donze}. Let $I$ be the identity matrix in $\R^{n\times n}$.  As we now show, the sensitivity matrix $\partial_x \chij(x)$ is a solution of the \emph{variational equations} associated with~\eqref{neuralode}:
\begin{equation}
\label{eq:variational}
\begin{aligned}[c]
        \partial_x\chij(x) = F(t_j,x)\hspace{12ex}\\
        \partial_t F(t,x) = (\partial_x f)(\chit(x)) F(t,x),\quad 
        F(t_0,x)=I
\end{aligned}
\end{equation}
%

\noindent
\emph{Proof sketch}:
 By interchanging the differentiation order, we obtain $\partial_t(\partial_x \chit(x))\,{=}\,\partial_x (\partial_t \chit(x))$. Since $\chit(x)$ is a solution of Eq.~\eqref{neuralode}, $\partial_x (\partial_t \chit (x))\,{=}\,\partial_x (f(\chit(x)))$. By the chain rule, we get $\partial_t (\partial_x \chit (x))\,{=}\,(\partial_x f)(\chit(x)) \partial_x \chit(x)$. 

\begin{algorithm}[t]
    \caption{Computation of $\nabla_\varphi L$}
    \label{algorithm:computing gradient}
    \begin{algorithmic}[1]
    \REQUIRE target time $t_j$, initial value $\varphi\in \RE^{n-1}$, Neural ODE $f$, gradients $\partial_x \Lc$ and $\partial_\varphi  x$
    \STATE $b \leftarrow  x(\varphi), F \leftarrow I$
    \STATE $[b,F] \leftarrow$ solve\_ivp($[f(b,t),(\partial_b f)(b)\cdot F],[0,t_j],[b,F])$
    \STATE $\nabla_\varphi L \leftarrow -\partial_y \Lc(y) \cdot F \cdot \partial_\varphi  x$
    \RETURN $\nabla_\varphi L$ \COMMENT{Required in line~\ref{line:loss gradient} of algorithm~\ref{algorithm:gradient descent}}
    \end{algorithmic}
\end{algorithm}
\noindent\textbf{Forward-mode use of adjoint sensitivity method.} The integral of Eq.~\eqref{eq:variational} has the same form of the auxiliary ODE used for reverse-mode automatic differentiation of Neural ODEs, when optimized by the adjoint sensitivity method \cite{neuralODEs} with one exception. In contrast to \cite{neuralODEs}, which requires one to run the adjoint equation backward and have access to the termination time of the flow, our approach enjoys a simultaneous forward-mode use of the adjoint equation. This is due to the way we determine the loss function in the ODE space. In retrospect, this enables us to obtain the gradients of the loss at the current state-computation step. This property enables us to improve the optimization runtime by 50\%, compared to the optimization scheme used in \cite{neuralODEs}: we save half of the time because we do not have to go backward to compute the loss.

More precisely, solving Eq.~\eqref{eq:variational} until target time $t_j$ requires knowledge of $\chit(x)$ for all $t\in[t_0,t_j]$. 
This ensures that we already know the value of $\chit(x_0)$ when needed to compute the right side of Eq.~\eqref{eq:variational} during integration of $F(t,x)$. Algorithm~\ref{algorithm:computing gradient} demonstrates the computation of the gradient $\nabla_\varphi L$ of the loss function.
%

\subsection{Safety-Region Computation}\label{sec:TR and TD}
With our global search strategy, we are covering the feasible region $\calB_0^S$ with already visited points $\calV$. Consequently, we have access to the global minimum in all of those regions:
\begin{align}\label{eq:local minimum}
    \bar{m} = \min_{\varphi\in\calV}L(\varphi)
\end{align}
with $\bar{m}\ge m^\star$, where $m^\star$ is the global minimum of Eq.~\eqref{eq:optim2}. We now identify safety regions for a Neural ODE flow and describe how this is incorporated in the SLR algorithm.  
\begin{definition}[Safety Region]\label{def:TR}
    Let $\varphi_i\,{\in}\,\calV\subseteq\RE^{n-1}$ be an already-visited point. A safety-radius $r_{\varphi_i}\,{=}\,r(\varphi_i)$ defines a \emph{safe spherical-cap} $B(\varphi_i,r_{\varphi_i})^S \,{=}\, B(x(\varphi_i),r_{\varphi_i}) \cap \calB_0^S$, because $L(\psi)\ge\mu\cdot\bar{m}$ for all $\psi$ s.t. $x(\psi)\in B(\varphi_i,r)^S$.
\end{definition}
Our objective is to use the local Lipschitz constants to define a radius $r_\varphi$ around an already visited point $\varphi$ s.t.\ we can guarantee that $B(\varphi,r_\varphi)^S$ is a safety region.
\begin{definition}[Lipschitz]\label{def:lipschitz}
    The local Lipschitz constant (LLC) of a function $L$ in a region $A$ is defined as a $\lambda_A \ge 0$ with
    \begin{align*}
        \|L(x)-L(y)\|\le \lambda_A \|x-y\| \quad\forall x,y\in A.
    \end{align*}
\end{definition}
In the following theorem, we use the LLC to define the radius $r_\varphi$ of the safety (or tabu) region $ B(\varphi,r_\varphi)^S$ around an already-visited point $\varphi\in\calV$.
\begin{theorem}[Radius of Safety Region]\label{thm:safety region radius}
    At target time $t_j$, let $\bar{m}$ be the current global minimum, as in Eq.~\eqref{eq:local minimum}.
    Let $\varphi\in\calV$ be an already-visited point with value $L(\varphi)$ ($\ge \bar{m}$), and let $r_\varphi$ and $ B(\varphi,r_\varphi)^S$ be defined as follows with $\mu\ge 1$:
    \begin{align}\label{eq:safety radius}
        r_{\varphi} =
        \lambda_{\Sigma_\varphi}^{-1}\left(L(\varphi)-\mu\cdot\bar{m}\right)
    \end{align}
    with $\lambda_{\Sigma_\varphi} =
    \max_{x(\psi)\in\Sigma_\varphi}\lVert \partial_x \chij(x(\psi)) \rVert_{M_{0,j}}$. If $\Sigma_\varphi$ is chosen s.t.\  $\Sigma_\varphi\supseteq  B(\varphi,r_\varphi)^S$, then it holds that:
    \begin{align}\label{eq:safety radius result}
        L(\psi)\ge \mu\cdot\bar{m}\quad\forall x(\psi)\in  B(\varphi,r_\varphi)^S
    \end{align}
\end{theorem}
%
%
The full proof is provided in the Appendix. \emph{Proof sketch:} The Lipschitz constant defines a relation between the values in the domain and the ones in the range of the function.
\begin{algorithm}[t]
    \caption{Computing the Radius of the Safety Region}
    \label{algorithm:radius for safety region}
    \begin{algorithmic}[1]
    \REQUIRE target time $t_j$, visited point $\varphi$, termination tolerance $\epsilon\,{>}\,0$, initial ball $\calB_0$ with radius $\rd_0$, minimum of visited points $\bar{m}$, loss function $L$, tolerance $\mu \,{\ge}\,1$, region $\Sigma_\varphi$ in which to compute the LLC $\lambda$.  
    \vspace*{2mm}
    \STATE $\Sigma_\varphi \leftarrow \calB_0, s\leftarrow \rd_0$
    \STATE $\lambda \leftarrow$ computeLipschitz($\Sigma_\varphi$)\label{line:compute lipschitz}
    \STATE $r \leftarrow 1/\lambda \cdot (L(\varphi)-\mu\cdot\bar{m})$
    \WHILE{$|r-s|/r > \epsilon$ \textbf{or} $s<r$}
        \STATE \textbf{set} $s \leftarrow r + |s-r|/2$
        \STATE $\Sigma_\varphi \leftarrow B(\varphi,s)^S$
        \STATE $\lambda \leftarrow$ computeLipschitz($\Sigma_\varphi$)
        \STATE $r \leftarrow 1/\lambda \cdot (L(\varphi)-\mu\cdot\bar{m})$
    \ENDWHILE
    \RETURN $r$
    \end{algorithmic}
\end{algorithm}

Theorem~\ref{thm:safety region radius} says that areas around already-visited samples are safe. The size of the safety areas increases if we have a better current global minimum. Therefore, the theorem demonstrates that we can improve the convergence rate if we optimize the loss, by possibly finding a better current global minimum. This justifies the use of gradient descent together with a more global search strategy.

%

Algorithm~\ref{algorithm:radius for safety region} computes the radius in Eq.~\eqref{eq:safety radius} as a fixpoint of the choice of $\Sigma_\varphi$. For an over-approximation of the LLC in Line~\ref{line:compute lipschitz}, we use the triangle inequality and the mean value inequality with a change in metric~\cite[Lemma 2]{Cyranka2017}.  We then solve Eq.~\ref{eq:variational} using interval arithmetic to obtain an interval gradient matrix $[\calF_t]\owns\partial_x\chij(x)$ $\forall\,x\in\,\Sigma_\varphi$, and take the maximum singular value of $[\calF_t]$, as proposed in~\cite{gruenbacherArch19}. Depending on the Neural ODE, it is presumably faster to pick $s\,{=}\,\rd_0$, and to always use the LLC $\lambda_{\calB_0}$ of the entire initial ball.
As a result of the way we select $r_\varphi$ in Theorem~\ref{thm:safety region radius}, we are able to increase the radii $r_{\varphi}$ 
as soon as a new region with a smaller local minimum than the previous ones is discovered. Thus: $\bar{m}\,{\le}\, \bar{m}_{prev}\Rightarrow L(\varphi)\,{-}\,\mu\cdot\bar{m} \ge L(\varphi)\,{-}\,\mu\bar{m}_{prev}\Rightarrow
r_\varphi \,{\ge}\, r_{\varphi, prev}$.

\begin{algorithm}[t]
    \caption{Stochastic Lagrangian Reachability}
    \label{algorithm:SLR}
    \begin{algorithmic}[1]
    \REQUIRE time horizon T, sequence of timesteps $t_j$ ($t_0\le t_1\le\dots\le t_k=T$), tolerance $\mu\,{\ge}\,1$, confidence level $\gamma\,{\in}\, (0,1)$, loss function $L$, gradient of loss $\nabla_\varphi L$
    \vspace*{2mm}
    \FOR{$(j=1; j\le k; j=j+1)$}\label{line:for loop Reachsets}
        \STATE $\calV, \calU \leftarrow \{\}$ \quad(list of visited and random points)
        \STATE $\calS\leftarrow \{\}$ \quad (total covered area)
        \STATE $\calP\leftarrow 0$, $\bar{m} \leftarrow 0$
        \WHILE{$\calP < 1 - \gamma$}
            \STATE \textbf{sample} $\varphi \in \R^{n-1}$
            \STATE $\calV\leftarrow \calV \cup \{\varphi\}$
            \STATE $\calU\leftarrow \calU \cup \{\varphi\}$
            \IF{$x(\varphi) \notin \calS$}
                \STATE $\varphi_{min} \leftarrow$ local minimum starting at $\varphi$ using gradient descent with $\nabla_\varphi L$\label{line:gradient descent}
                \STATE $\calV\leftarrow \calV \cup \{\varphi_{min}\}$
                \STATE $m \leftarrow L(\varphi_{min})$
            \ELSE
                \STATE $m \leftarrow L(\varphi)$
            \ENDIF
            \IF{$m \le \bar{m}$}
                \STATE $\bar{m} \leftarrow m$
                \STATE \textbf{set} $S\leftarrow\{\}$
                \FORALL{$\varphi_i\in \calV$}
                    \STATE \textbf{compute} new radius $r=r(\varphi_i)$ such that $L(\psi)\ge\mu\cdot\bar{m}$, \quad $\forall\psi\colon x(\psi)\in B(\varphi_i,r)^S$\label{line:increase radii}
                    \STATE \textbf{set} $\calS\leftarrow \calS\cup B(\varphi_i,r)^S$
                \ENDFOR
            \ELSE
                \STATE \textbf{compute} radius $r=r(\varphi)$ only for current $\varphi$ such that $L(\psi)\ge\mu\cdot\bar{m}$, \quad $\forall\psi\colon x(\psi)\in B(\varphi,r)^S$
                \STATE \textbf{set} $\calS\leftarrow \calS\cup B(\varphi,r)$
            \ENDIF
            \STATE $\calP \leftarrow \Pr(\mu \cdot \bar{m} \le m^\star)$ with $\mu\cdot\bar{m}\le\min_{\varphi\in\calS} L(\varphi)$
        \ENDWHILE
        \STATE $\rd_j\leftarrow -\bar{m}$
    \ENDFOR
    \RETURN $(\rd_1,\dots,\rd_k)$
    \end{algorithmic}
\end{algorithm}
\subsection{Stochastic Lagrangian Reachability}
By using local gradient computation, global uniform sampling, and safety regions as in Algorithm~\ref{algorithm:radius for safety region}, we present our SLR verification technique, as outlined in Algorithm~\ref{algorithm:SLR}. 

Given a Neural ODE as in Eq.~\eqref{neuralode} and a set of initial states $\calB_0$, we start by specifying a confidence level $\gamma\in(0,1)$ and a tolerance $\mu,{\ge}\,1$ for the entire Reachtube. The algorithm returns radii $\rd_j$,  $j\,{\in}\,\{1,\dots,k\}$, and the stochastic guarantee stating that $\calB_j (=B_{M_j}(\chij,\rd_j))$ overestimates by $\mu$ the true conservative Reachsets with a probability higher than $1\,{-}\,\gamma$. This holds also for the whole Reachtube, as it is defined by a series of Reachsets (Def.~\ref{def:Reachtube}). 

As we reinitialize the variables at the beginning of every new timestep $t_j$, and apply gradient descent to the loss function of the initial polar coordinates $\varphi$ at time $t_0$, we do not accumulate errors from one timestep to the next one. This is a prominent advantage compared to methods using interval arithmetic, and thus accumulating the wrapping effect, e.g.~\cite{capdTheory, CyrankaCDC18,fansimul}. Another advantage is that we can compute the for-loop in line~\ref{line:for loop Reachsets} of Algorithm~\ref{algorithm:SLR} (thus the Reachsets of the Reachtube) in parallel.

At every timestep $t_j$, we sample random points and construct safety regions around them until we reach the desired probability $1\,{-}\,\gamma$ of being inside the tolerance region defined by $\mu$. After sampling a new point, we check if this point is already in the covered area. If not, then we apply gradient descent to find a local minimum and compare this local minimum to the smallest value $\bar{m}$. Otherwise, if the sampled point is already in the covered area and thus in at least one safety region, we already know the lower bounds for that region and do not look for the local minimum again. This approach is similar to using baisins of attraction, but is more scalable because we do not require Hessian computation. In line~\ref{line:increase radii}, we recompute the radii of the safety regions when we find a new smallest value $\bar{m}$. By computing the current probability $\bar{p}$ of having reached the desired confidence level, we check whether we have to resample more points or whether we are able to finish that timestep and save the radius $\rd_j$ of the stochastic Reachset at time $t_j$.
\subsection{Stochastic Guarantees of Reachsets}\label{stochastic}
In this section, we derive the stochastic convergence guarantees and convergence bounds for finding the global minimum of Eq.~\eqref{eq:optim2} using SLR at every timestep $t_j$.
\begin{figure}[t]
    \centering
    \includegraphics[width=0.7\columnwidth]{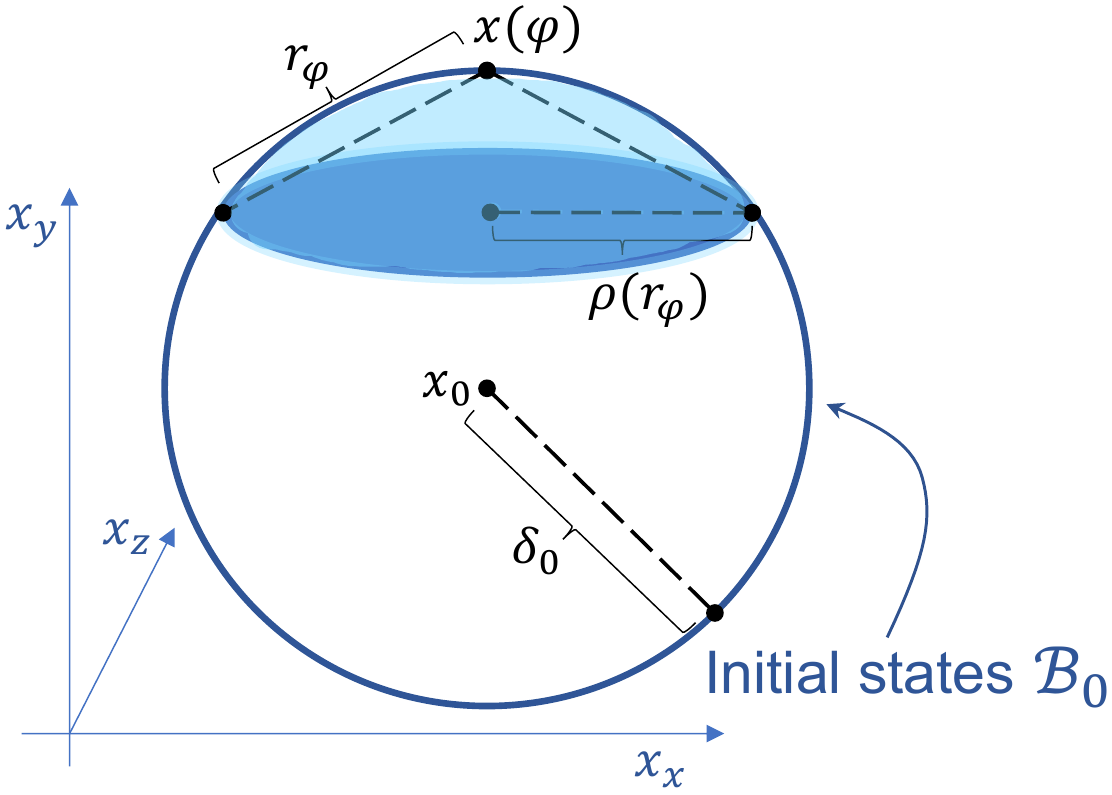}
    \caption{Illustration of a safety region $B(\varphi,r_\varphi)^S$, which is a spherical cap $\Ystar$. In this figure, the area of cap $\Ystar$ (in light blue) is greater than the volume of an $n-1$-dimensional ball (in dark blue) with radius $\rho(r_\varphi)$, which is used in the convergence rate.
}
    \label{fig:Probability1}
\end{figure}

Let $\bar{m}\,{=}\, \min_{\varphi\in\calV}L(\varphi)$ be defined as in Eq.~\eqref{eq:local minimum}, and let $m^\star\,{=}\,\min_{\varphi\in \RE^{n-1}} \Lp(\varphi)$ be the global minimum and $\varphi^\star$ an argument s.t.\  $\Lp(\varphi^\star)=m^\star$.
We start by defining the probability of $\Bphi$ covering $x(\phistar)$:
\begin{align}
    &\Pr(\Bphi\owns x(\varphi^\star))\nonumber\\
    &= \Pr\left(\| x(\varphi^\star)- x(\varphi)\|_2)\le \rmu\right)\nonumber\\
    &= \Pr(x(\varphi)\in \Ystar) = \Pr(\Ystar)\label{eq:probability region}
\end{align}
with $\rmu$ as defined in Eq.~\eqref{eq:safety radius} and $\Ystar =  B(\phistar,r_\varphi)^S$ being the spherical-cap in Fig.~\ref{fig:Probability1}.
By using the area of the spherical cap $\Ystar$ and the area of the initial ball's surface $\calB_0^S$, the probability defined by Eq.~\eqref{eq:probability region} can be described as follows:
\begin{align}
    \Pr(\Ystar) = \frac{\area(\Ystar)}{\area(\calB_0^S)}
\end{align}
The area of $\Ystar$ can be computed using the formulas in~\cite{hypersphericalCap}.
Next we derive some probabilities:
\begin{align}
    \Pr(\Bphi[\varphi_j]\niton\phistar) &= 1 - \Pr(\Ystar[\varphi_j])\nonumber\\
    \Pr(\forall \varphi\in\calU\colon \Bphi[\varphi]\niton\phistar) &= \prod_{\varphi\in\calU} \left(
    1 - \Pr(\Ystar[\varphi])
    \right)\nonumber\\
    \Pr(\exists \varphi\in\calU\colon\Bphi[\varphi]\owns\varphi^\star) &= 1 - \prod_{\varphi\in\calU} \left(
    1 - \Pr(\Ystar[\varphi])
    \right)\label{eq:probability of finding min}
\end{align}
Using Theorem~\ref{thm:safety region radius}, if $\phistar\in\Bphi[\varphi]$ for some $\varphi\in\calU$, then $\mu\cdot\bar{m}\le L(\phistar) = m^\star$ holds, and thus:
    \begin{align}\label{eq:probability of mu}
        \begin{split}
            &\Pr(\mu\cdot\bar{m}\le m^\star) \ge\\
            & \Pr(\exists \varphi\in\calU\colon\Bphi[\varphi]\owns\varphi^\star) 
        \end{split}
    \end{align}

\begin{theorem}[Convergence Guarantees]\label{thm:convergence guarantee}
    Given $\gamma\in(0,1)$, $\mu\ge 1$, local Lipschitz constant $\lambda_{\calB_0^S}$ and $N = |\calU|$, where N is the number of uniform-randomly generated points during global search process. Let $\bar{m}\,{=}\, \min_{\varphi\in\calV}L(\varphi)$ as defined in Eq.~\eqref{eq:local minimum}, $m^\star\,{=}\,\min_{\varphi\in \RE^{n-1}} \Lp(\varphi)$ the global minimum, and $\varphi^\star$ an argument s.t.\  $\Lp(\varphi^\star)=m^\star$. Then:
    \begin{align}
        \lim_{N\rightarrow\infty}\Pr(\mu\cdot\bar{m}_N\le m^\star) = 1\label{eq:convergence}
    \end{align}
    and thus
    \begin{align}
        \forall\gamma\in(0,1),\exists N\in\mathbb{N}\textrm{ s.t. }
        \Pr(\mu\cdot\bar{m}_N\le m^\star) \ge 1-\gamma
    \end{align}
\end{theorem}
\vspace{2ex}
The full proof is provided in the Appendix. \emph{Proof sketch:} By creating a lower bound $r_{bound}$ for all $\rmu$, s.t.\  $\Pr(\calC(\rmu)\ge\Pr(\calC(r_{bound}))$, we underestimate Eq.~\eqref{eq:probability of finding min} by $1-(1-\Pr(\calC(r_{bound})))^N$. Using this bound and Eq.~\eqref{eq:probability of mu}, we show that the convergence guarantee holds.

Theorem \ref{thm:convergence guarantee} shows that in the limit of the number of samples, the reachset constructed by Algorithm~\ref{algorithm:SLR} converges with probability~1 to the smallest ellipsoid that encloses the true reachable set. Note that the algorithm cannot converge to the true reachable set because we approximate the reachset by ellipsoids, while the true reachset might be of arbitrary geometrical shape. Nonetheless, we proved that it provides the smallest possible ellipsoid that contains a true reachset. 

Moreover, although Theorem~\ref{thm:convergence guarantee} shows that we achieve the tightest elliptical reachsets, it does not determine whether the algorithm can terminate or not, as the theorem is proven in the case of infinite samples. We now prove that SLR indeed converges at a reasonable rate. 

\subsection{Convergence Rate for SLR}
Theorem~\ref{thm:convergence rate} computes a convergence rate for Algorithm~\ref{algorithm:SLR}. 
\begin{theorem}[Convergence Rate]\label{thm:convergence rate}
    Given $\gamma\in(0,1)$, $\mu\ge 1$, local Lipschitz constant $\lambda_{\calB_0^S}$, and dimension $n$, let $\varphi_1$ be the first random sample point. We can guarantee that $\Pr(\mu\cdot\bar{m}\le m^\star)\ge 1-\gamma$ if we perform at most $N_{max}$ iterations of the SLR Algorithm~\ref{algorithm:SLR}, with
    \begin{equation}
    \begin{split}
        &N_{max} = \\
        &\ln{\gamma}\left/\ln\left(1-\frac1{2\sqrt{\pi}}\frac{\Gamma(n/2)}{\Gamma((n+1)/2)}\left(\frac{\rho(r_{bound})}{\rd_0}\right)^{n-1}\right)\right.\label{eq:guarantee maximum N}
    \end{split}
    \end{equation}
    and asymptotically it holds that
    \begin{align}
        N_{max} = \mathcal{O}\left(-\ln\gamma \left(\frac{\rd_0}{r_{bound}}\right)^{2n}\right),\label{eq:guarantee asymptotically}
    \end{align}
    with $r_{bound} = \lambda_{\calB_0^S}^{-1}(1-\mu)L(\varphi_1)$ and $\rho(r_{bound}) = r_{bound} \cdot\sin(\pi/2 - \arcsin(r/2\rd_0))$.
\end{theorem}
The full proof is provided in the Appendix. \emph{Proof sketch:} As the radius $\rmu$ of the spherical cap is very small, we underestimate the area of the cap by removing the curvature and using the volume of an $n-1$ dimensional ball with radius $\rho(r_{bound})$ as shown in Fig.~\ref{fig:Probability1}. 
Thus, after finishing our global search strategy for timestep $t_j$, we have the stochastic guarantee that the functional values of every $\varphi\in\RE^{n-1}$ are greater or equal to $\mu\cdot\bar{m}$. This implies that we should initiate the search with a relatively large $\mu=\mu_1$, obtaining for every $\varphi$ a relatively large value of $r_{\varphi,\mu_1}$ and therefore obtain a faster coverage of the search space. Subsequently, we can investigate whether the reachset $\calB_j$ with radius $\rd_j=-\mu_1\cdot\bar{m}$ intersects with a region of bad (unsafe) states. If this is not the case, we can proceed to the next timestep $t_{j+1}$. Otherwise, we reduce $\mu$ to $\mu_2 < \mu_1$, which reduces the safety regions $\Bphi$ and thus the already-covered-set $\calS$. This means that we continue with our search strategy until the desired probability $1-\gamma$ is reached again for a smaller radius $\rd_j=-\mu_2\cdot\bar{m}$. Accordingly, we can find a first radius for $\calB_j$ faster and refine it as long as $\calB_j$ intersects with the region of bad states.

Theorem~\ref{thm:convergence rate} guarantees convergence of the algorithm. It shows that for a given confidence level $\gamma$, our algorithm terminates after at most $N_{max}$ steps. Essentially, the theorem leads us to the significant result that the problem of constructing an ellipsoid abstraction of the true reachset with probabilistic guarantees for a Neural ODE is able to terminate.


Additionally, the theorem assumes that we know the local Lipschitz constant, which is a reasonable assumption for proving convergence. In practice, one can safely replace the true Lipschitz constant by an upper-bound.

\subsection{Computational Complexity}
The complexity of Algorithm~\ref{algorithm:gradient descent} depends on the geometry of the loss surface. In particular, Algorithm~\ref{algorithm:gradient descent} may terminate after one iteration in case of a flat surface, whereas an exponential number may be needed for ill-posed problems, as is common practice when deriving convergence rates for gradient descent \cite{nagy2003steepest, drori2017exact}

The runtime of Algorithm~\ref{algorithm:computing gradient} is determined by the complexity of the ODE solver for simulating the given differential equation. For example, given the number of integration steps (implicit interpretation of the number of layers in a deep model) $L$, and the time horizon of the simulation $T$, Algorithm~\ref{algorithm:computing gradient} runs in time $\mathcal{O}(L \times T)$ and constant memory cost $\mathcal{O}(1)$ for each layer of a neural network $f$. 

The complexity of Algorithm~\ref{algorithm:radius for safety region} depends on the local Lipschitz constant and the smoothness of the flow. Computing the true Lipschitz constant of a neural network is known to be NP-complete \cite{virmaux2018lipschitz}. However, Algorithm~\ref{algorithm:radius for safety region} operates correctly when we replace the true Lipschitz constant by an easier-to-compute upper bound, obtained for instance by means of interval arithmetic.

Algorithm~\ref{algorithm:SLR} implements the main routine of our framework. Its complexity for a given confidence score $\gamma$ equals the convergence rate $N_{max}$ proven in Theorem~\ref{thm:convergence rate}, Eq.~\eqref{eq:guarantee asymptotically} 
for every Reachset. In particular, the runtime of Algorithm~\ref{algorithm:SLR} depends exponentially on the dimension of the given Neural ODE and logarithmically on the confidence score. 

\section{Conclusions and Future Work}\label{conclusions}
In this paper, we considered the verification problem for Neural ODEs. We introduced the SLR verification scheme, which is based on solving a global optimization problem. We designed a forward formulation of the adjoint method for the gradient descent algorithm.  We also established strong convergence guarantees for SLR, showing that it can establish tight ellipsoidal bounds for the Neural ODE under consideration, at an arbitrary time horizon.

An important future direction will be to improve the current convergence rate, which is exponential in the dimensionality of the Neural ODE network. Existing statistical verification methods are mostly concerned with the verification of (hybrid) dynamical systems having various uncertainties in model parameters, discrete jumps between modes, and/or initial states. We emphasize that reachability computation for Neural ODEs developed at scale will require dedicated methods tailored for that specific purpose.

\section*{Acknowledgements}
The authors would like to thank the reviewers for their insightful comments.
RH and RG were partially supported by Horizon-2020 ECSEL Project grant No. 783163 (iDev40). RH was partially supported by Boeing. ML was supported in part by the Austrian Science Fund (FWF) under grant Z211-N23 (Wittgenstein Award). SG was funded by FWF project W1255-N23. JC was partially supported by NAWA Polish Returns grant PPN/PPO/2018/1/00029.  SS was supported by NSF awards DCL-2040599, CCF-1918225, and CPS-1446832.

%
%

\bibliography{LRT_QR}
\newpage
~
\newpage

\appendix
\section{Appendix}
    \begin{theorem}[Radius of Safety Region]\label{thm-appendix:safety region radius}
    At target time $t_j$, let $\bar{m}$ be the current global minimum $\bar{m} = \min_{\varphi\in\calV}L(\varphi)$.
    Let $\varphi\in\calV$ be an already visited point with value $L(\varphi)$ ($\ge \bar{m}$) and let $r_\varphi$ and $ B(\varphi,r_\varphi)^S$ be defined as follows with $\mu\ge 1$:
    \begin{align}\label{eq-appendix:safety radius}
    \begin{split}
        r_{\varphi} &=
        \lambda_{\Sigma_\varphi}^{-1}\left(L(\varphi)-\mu\cdot\bar{m}\right)
    \end{split}
    \end{align}
    with $\lambda_{\Sigma_\varphi} =
    \max_{x(\psi)\in\Sigma_\varphi}\lVert \partial_x \chij(x(\psi)) \rVert_{M_{0,j}}$. If $\Sigma_\varphi$ is chosen s.t. $\Sigma_\varphi\supseteq  B(\varphi,r_\varphi)^S$, then it holds that
    \begin{align}\label{eq-appendix:safety radius result}
        L(\psi)\ge \mu\cdot\bar{m}\quad\forall x(\psi)\in  B(\varphi,r_\varphi)^S
    \end{align}
\end{theorem}
%
%
\begin{proof}
Given $r_\varphi$ and $\lambda_{\Sigma_\varphi}$ as defined in the above theorem. Using the mean value inequality for vector valued functions, the triangle inequality, and considering the change of metric~\cite[Lemma 2]{Cyranka2017} it holds that:
\begin{align*}
    &|L(\varphi_1)-L(\varphi_2)| = \\
    & \left| \left\lVert \chij(x(\varphi_1)) - \chij(x_0)\right\rVert_{M_j} - \left\lVert \chij(x(\varphi_2)) - \chij(x_0)\right\rVert_{M_j}\right| \\
    &\le \left\lVert \chij(x(\varphi_1)) - \chij(x(\varphi_2))\right\rVert_{M_j} \\
    &\le \lambda_{\Sigma_\varphi} \left\lVert x(\varphi_1) - x(\varphi_2) \right\rVert_I \quad \\  
    &\forall x(\varphi_1),x(\varphi_2) \in \Sigma_\varphi\supseteq  B(\varphi,r_\varphi)^S
\end{align*}
Thus $\lambda_{\Sigma_\varphi}$ is a local Lipschitz constant in the safety region $ B(\varphi,r_\varphi)^S$. By definition $\|x(\psi)\,{-}\,x(\varphi) \| \,{\le}\, r_\varphi$ for $x(\psi),{\in}\, B(\varphi,r_\varphi)^S$. Hence:
\begin{align}\label{eq-appendix:proof lipschitz}
    &|L(\psi)-L(\varphi)|\\
    &\le \lambda_{\Sigma_\varphi} \|x(\psi) - x(\varphi)\| \\
    &\le \lambdaS r_\varphi = L(\varphi)-\mu\cdot\bar{m} \quad
    \forall x(\psi)\in B(\varphi,r_\varphi)^S
\end{align}
To prove that Eq.~\eqref{eq-appendix:safety radius result} holds, we distinguish between two cases for $\psi$: (1)~$L(\psi)\,{\ge}\,L(\varphi)$ and (2)~$L(\psi)\,{<}\,L(\varphi)$. Case~(1) it is straightforward: $L(\psi)\ge L(\varphi)\ge \mu\cdot L(\varphi) \ge \mu\cdot\bar{m}$. In case~(2), we use Eq.~\eqref{eq-appendix:proof lipschitz} and thus:
\begin{align*}
    |L(\psi) - L(\varphi)| &= L(\varphi)-L(\psi) \le L(\varphi)-\mu\cdot\bar{m} \\
    &\Longrightarrow L(\psi) \ge \mu\cdot\bar{m},
\end{align*}
proving that Eq.~\eqref{eq-appendix:safety radius result} holds no matter if $L(\psi)\,{\ge}\, L(\varphi)$ or if $L(\psi) \,{<}\, L(\varphi)$, for all $x(\psi) \in B(\psi,r_\psi)^S$.
\end{proof}
\begin{theorem}[Convergence Guarantees]\label{thm-appendix:convergence guarantee}
    Given $\gamma\in(0,1)$, $\mu\ge 1$, local Lipschitz constant $\lambda_{\calB_0^S}$ and $N = |\calU|$, where N is the number of uniform-randomly generated points during global search process. Let $\bar{m}\,{=}\, \min_{\varphi\in\calV}L(\varphi)$ be the current minimum, $m^\star\,{=}\,\min_{\varphi\in \RE^{n-1}} \Lp(\varphi)$ the global minimum, and $\varphi^\star$ an argument s.t. $\Lp(\varphi^\star)=m^\star$. It holds that
    \begin{align}
        \lim_{N\rightarrow\infty}\Pr(\mu\cdot\bar{m}_N\le m^\star) = 1\label{eq-appendix:convergence}
    \end{align}
    and thus
    \begin{align}
        \forall\gamma\in(0,1)\exists N\in\mathbb{N},\textrm{ s.t. }
        \Pr(\mu\cdot\bar{m}_N\le m^\star) \ge 1-\gamma
    \end{align}
\end{theorem}
\vspace{2ex}
\begin{proof}
    Next we derive some probabilities:
    \begin{align}
        \Pr(\Bphi[\varphi_j]\niton\phistar) &= 1 - \Pr(\Ystar[\varphi_j])\nonumber\\
        \Pr(\forall \varphi\in\calU\colon \Bphi[\varphi]\niton\phistar) &= \prod_{\varphi\in\calU} \left(
        1 - \Pr(\Ystar[\varphi])
        \right)\nonumber\\
        \Pr(\exists \varphi\in\calU\colon\Bphi[\varphi]\owns\varphi^\star) &= 1 - \prod_{\varphi\in\calU} \left(
        1 - \Pr(\Ystar[\varphi])
        \right)\label{eq-appendix:probability of finding min}
    \end{align}
    Using Theorem~\ref{thm-appendix:safety region radius}, if $\phistar\in\Bphi[\varphi]$ for some $\varphi\in\calU$, then $\mu\cdot\bar{m}\le L(\phistar) = m^\star$ holds, thus:
        \begin{align}\label{eq-appendix:probability of mu}
            \begin{split}
                &\Pr(\mu\cdot\bar{m}\le m^\star) \ge\\
                & \Pr(\exists \varphi\in\calU\colon\Bphi[\varphi]\owns\varphi^\star) 
            \end{split}
        \end{align}
    Thus, it holds that $\Pr(\mu\cdot\bar{m}\le m^\star)\ge 1 - \prod_{\varphi\in\calU} \left(
    1 - \Pr(\Ystar[\varphi])\right)$, with $\rmu$ as defined in Eq.~\eqref{eq-appendix:safety radius}.
    \begin{align}
        \rmu &\ge \lambda_{\calB_0^S}^{-1} (\min_{\varphi\in\calU}L(\varphi) - \mu\cdot\bar{m})\\
        &\ge \lambda_{\calB_0^S}^{-1}(1-\mu)\bar{m}\\
        &\ge \lambda_{\calB_0^S}^{-1}(1-\mu)L(\varphi_1)=r_{bound}\quad\forall \varphi\in\calU\label{eq-appendix:r bound},
    \end{align}
    with $\varphi_1$ being the first random sampled point.
    Hence:
    \begin{align}
        \Pr(\mu\cdot\bar{m}\le m^\star)&\ge 1 - \prod_{\varphi\in\calU} \left(
        1 - \Pr(\Ystar[{bound}])\right)\\
        &\ge 1 - \left(
        1 - \Pr(\Ystar[{bound}])\right)^N\label{eq-appendix:bound of probability}
    \end{align}
    As $\Pr(\Ystar[{bound}])\in(0,1)$, it follows that Eq.~\eqref{eq-appendix:convergence} holds and thus we are able to guarantee the convergence of our global search strategy.
\end{proof}
\begin{theorem}[Convergence Rate]\label{thm-appendix:convergence rate}
    Given $\gamma\in(0,1)$, $\mu\ge 1$, local Lipschitz constant $\lambda_{\calB_0^S}$ and dimension $n$. Let $\varphi_1$ be the first random-samples point. We can guarantee that $\Pr(\mu\cdot\bar{m}\le m^\star)\ge 1-\gamma$ if we perform at most $N_{max}$ iterations of the SLR Algorithm, with
    \begin{equation}
    \begin{split}
        &N_{max} = \\
        &\ln{\gamma}\left/\ln\left(1-\frac1{2\sqrt{\pi}}\frac{\Gamma(n/2)}{\Gamma((n+1)/2)}\left(\frac{\rho(r_{bound})}{\rd_0}\right)^{n-1}\right)\right.\label{eq-appendix:guarantee maximum N}
    \end{split}
    \end{equation}
    and asymptotically it holds that
    \begin{align}
        N_{max} = \mathcal{O}\left(-\ln\gamma \left(\frac{\rd_0}{r_{bound}}\right)^{2n}\right),\label{eq-appendix:gurantee asymptotically}
    \end{align}
    with $r_{bound} = \lambda_{\calB_0^S}^{-1}(1-\mu)L(\varphi_1)$ (as defined in Eq.~\eqref{eq-appendix:r bound}) and $\rho(r_{bound}) = r_{bound} \cdot\sin(\pi/2 - \arcsin(r/2\rd_0))$.
\end{theorem}
\begin{proof}
    If the right-hand side of Eq.~\eqref{eq-appendix:bound of probability} equals $1-\gamma$, it holds that $\Pr(\mu\cdot\bar{m}\le m^\star)\ge 1-\gamma$. Thus reformulating that equation, we get an upper bound $N$:
    \begin{align}
        1 - \left(
        1 - \Pr(\Ystar[{bound}])\right)^{N} &= 1 - \gamma & \Leftrightarrow\\
        \left(
        1 - \Pr(\Ystar[{bound}])\right)^{N} &= \gamma & \Leftrightarrow\\
        \frac{\ln(\gamma)}{\ln(1-\Pr(\Ystar[{bound}]))}&= N &\label{eq-appendix:bound N}
    \end{align}
    We use the following to overestimate of $\Pr(\Ystar[{bound}])$: Using the results of~\cite{hypersphericalCap} it holds that $\area(\Ystar[{bound}])\ge \vol_{n-1}(\rho(r_{bound}))$, thus:
    \begin{align}\label{eq-appendix:probability of cap}
        \begin{split}
            &\Pr(\Ystar[{bound}])\ge \frac{\textrm{Vol}_{n-1}(\rho(r_{bound}))}{\area(\calB_0)}=\\
            &\frac{\pi^{(n-1)/2}}{\Gamma((n+1)/2)}\rho(r_{bound})^{n-1}\frac{\Gamma(n/2)}{2\pi^{n/2}}\frac1{\rd_0^{n-1}}
        \end{split}
    \end{align}
    with $\Gamma$ being the gamma function. Putting Eq.~\eqref{eq-appendix:probability of cap} and Eq.~\eqref{eq-appendix:bound N} we obviously get to equation~\eqref{eq-appendix:guarantee maximum N}. Using the information that $\ln(1-x)\approx -x$ and using the results of ~\cite{stochGlobOptim}[Section 2.2], the asymptotical value of $N_{max}$ equals to Eq.~\eqref{eq-appendix:gurantee asymptotically}
\end{proof}

\end{document}